\documentclass[sigconf]{acmart}
\AtBeginDocument{%
  }

\copyrightyear{2025}
\acmYear{2025}
\setcopyright{cc}
\setcctype{by}
\acmConference[DAI '25]{The Seventh International Conference on Distributed Artificial Intelligence}{November 21--24, 2025}{London, United Kingdom}
\acmBooktitle{The Seventh International Conference on Distributed Artificial Intelligence (DAI '25), November 21--24, 2025, London, United Kingdom}
\acmPrice{}
\acmDOI{10.1145/3772429.3772496}
\acmISBN{979-8-4007-2275-2/25/11}

\usepackage[ruled,vlined]{algorithm2e}
\SetKwInput{KwInput}{Input}
\SetKwInput{KwOutput}{Output}
\usepackage{subfigure}
\usepackage{multirow}

\begin{document}


\title{ARAC: Adaptive Regularized Multi-Agent Soft Actor-Critic in Graph-Structured Adversarial Games}

\author{Ruochuan Shi}
\affiliation{
  \institution{Institute of Automation, Chinese Academy of Sciences}
  \city{Beijing}
  \country{China}
}
\affiliation{
  \institution{School of Artificial Intelligence, University of Chinese Academy of Sciences}
  \city{Beijing}
  \country{China}
}
\email{shiruochuan2025@ia.ac.cn}

\author{Runyu Lu}
\affiliation{
  \institution{School of Artificial Intelligence, University of Chinese Academy of Sciences}
  \city{Beijing}
  \country{China}
}
\affiliation{
  \institution{Institute of Automation, Chinese Academy of Sciences}
  \city{Beijing}
  \country{China}
}
\email{lurunyu17@mails.ucas.ac.cn}

\author{Yuanheng Zhu}
\authornote{Corresponding Authors}
\affiliation{
  \institution{Institute of Automation, Chinese Academy of Sciences}
  \city{Beijing}
  \country{China}
}
\affiliation{
  \institution{School of Artificial Intelligence, University of Chinese Academy of Sciences}
  \city{Beijing}
  \country{China}
}
\email{yuanheng.zhu@ia.ac.cn}

\author{Dongbin Zhao}
\authornotemark[1]
\affiliation{
  \institution{Institute of Automation, Chinese Academy of Sciences}
  \city{Beijing}
  \country{China}
}
\affiliation{
  \institution{School of Artificial Intelligence, University of Chinese Academy of Sciences}
  \city{Beijing}
  \country{China}
}
\email{dongbin.zhao@ia.ac.cn}


\begin{abstract}

In graph-structured multi-agent reinforcement learning (MARL) adversarial tasks such as pursuit and confrontation, agents must coordinate under highly dynamic interactions, where sparse rewards hinder efficient policy learning. We propose \textit{Adaptive Regularized Multi-Agent Soft Actor-Critic} (ARAC), which integrates an attention-based graph neural network (GNN) for modeling agent dependencies with an adaptive divergence regularization mechanism. The GNN enables expressive representation of spatial relations and state features in graph environments. Divergence regularization can serve as policy guidance to alleviate the sparse reward problem, but it may lead to suboptimal convergence when the reference policy itself is imperfect. The adaptive divergence regularization mechanism enables the framework to exploit reference policies for efficient exploration in the early stages, while gradually reducing reliance on them as training progresses to avoid inheriting their limitations. Experiments in pursuit and confrontation scenarios demonstrate that ARAC achieves faster convergence, higher final success rates, and stronger scalability across varying numbers of agents compared with MARL baselines, highlighting its effectiveness in complex graph-structured environments.
\end{abstract}




\maketitle

\section{Introduction}

Multi-Agent Reinforcement Learning (MARL) has recently attracted significant attention in various domains such as multi-agent cooperation \cite{chai2024aligning}, competitive games \cite{vinyals2019grandmaster}, traffic control\cite{wiering2000multi}, and resource allocation \cite{marino2025decentralized}. In real-world environments, the interactions among multiple agents are often constrained by geographical structures, obstacle distributions, and communication ranges. Directly modeling such environments in continuous space (\cite{oyler2016pursuit, liang2019differential, garcia2020multiple}) not only leads to high computational complexity but also makes it difficult to capture both local relationships and global constraints among agents. In contrast, modeling the environment as a graph provides a more intuitive representation of the agents' reachable positions, feasible movement paths, and mutual visibility or interaction relationships. Graph nodes can represent locations or regions where agents may reside, while edges describe the movability or interaction reachability between these nodes. The presence of obstacles can be modeled by removing the corresponding edges. This graph-based representation reduces the complexity of environment modeling and facilitates the application of graph algorithms for path computation, reachability analysis, and other tasks.

In graph-structured environments, efficiently extracting structural features of agents and their neighborhoods is key to improving MARL performance \cite{munikoti2023gnnmarl}. Graph Neural Networks (GNNs) have shown strong capability in modeling structured data, effectively capturing both local and global topological information through message passing mechanisms. When combined with MARL, GNNs can provide each agent with a structured state representation, thereby improving policy generalization and adaptability to complex graph structures  \cite{JiangDHL20DGN, mcclellan2024E2GN2, LiLWC0MC024Grasper, lu2025equilibrium}. However, commonly used methods such as GCN \cite{KipfW17GCN} and GAT \cite{velivckovic2017graph}, while capable of aggregating neighborhood information, often face limitations in highly dynamic and many-to-many interaction scenarios. Specifically, their performance can be hindered by restricted information propagation ranges and limited adaptability in attention weight allocation, making it difficult to fully capture critical interaction patterns.

On the other hand, in sparse-reward tasks, MARL typically suffers from low exploration efficiency, with the training process prone to becoming stuck in local optima. Reference policy guidance has been widely adopted in reinforcement learning to address this issue \cite{hare2019dealing, GoecksGLVW20Integrating}. The core idea is to introduce a divergence term (e.g., Kullback–Leibler divergence) between the current policy and a reference policy into the loss function, forming a divergence-regularized policy optimization framework. This approach can accelerate policy convergence in the early training stage, but its effectiveness heavily depends on the quality of the reference policy and the choice of the regularization coefficient. In particular, in multi-agent settings, a fixed regularization coefficient often requires extensive tuning across different environments, and when the reference policy is suboptimal, long-term reliance can constrain final performance.

In this work, we introduce a MARL framwork \textit{Adaptive Regularized Multi-Agent Soft Actor-Critic} (ARAC) for multi-agent adversarial games. 
The main contributions of this paper are summarized as follows:
\begin{itemize}
    \item \textbf{Graph-structured environment modeling and feature representation:} We propose constructing a global state feature matrix by combining shortest-path distances, potential attack damage, agent health points (HP), and survival status, and design an attention-based encoder–decoder GNN architecture to enhance agents' perception and decision-making in complex graph environments.
    \item \textbf{Adaptive divergence-regularized policy optimization:} We develop an adaptive divergence regularization mechanism within the ARAC framework, which enables the algorithm to effectively exploit reference policy guidance in the early stage of training while gradually reducing reliance on potentially suboptimal reference policies in later stages.  
    \item \textbf{Extensive multi-scenario evaluation:} We conduct comprehensive experiments on both pursuit and confrontation multi-agent tasks, comparing ARAC with multiple baselines, and further assess its scalability under different numbers of agents.
\end{itemize}

\section{Related Work}

\subsection{Multi-Agent Reinforcement Learning}
MARL aims to learn optimal policies in environments involving multiple agents through interaction with both the environment and other agents. In cooperative, competitive, and mixed tasks, MARL must address challenges such as state dependencies among agents, the exponential growth of the joint action space, and the non-stationarity caused by concurrent learning. Classical MARL methods include value-decomposition-based algorithms such as QMIX \cite{rashid2020monotonic} and Value Decomposition Networks (VDN) \cite{sunehag2018value} for cooperative tasks, as well as policy-gradient-based methods for more general scenarios.

Soft Actor-Critic (SAC) \cite{haarnoja2018soft,haarnoja2018soft2} is an off-policy deep reinforcement learning algorithm that incorporates the maximum entropy principle, offering strong exploration capability and stable convergence. Some previous work \cite{iqbal2019actor, pu2021decomposed, hu2023graph} has extended SAC to multi-agent settings. However, in sparse-reward environments, even MASAC may suffer from inefficient exploration and slow convergence, motivating the incorporation of external guidance or reward shaping to enhance performance.

\subsection{Regularized Policy Optimization}

Regularized policy optimization enhances policy learning by introducing constraints or regularization terms that bias the learned policy towards a reference distribution, dataset, or prior knowledge. This paradigm has been widely applied across different contexts in reinforcement learning.

\textbf{Behavior Cloning (BC).}  
BC is a supervised learning approach where the policy directly imitates the action distribution of an expert or reference policy \cite{zare2024IL}. By minimizing the cross-entropy or KL divergence between the current policy and the reference, BC can accelerate policy initialization and improve early-stage convergence, provided the reference policy is of reasonable quality and data coverage is sufficient. However, if the reference policy is suboptimal or the data distribution differs significantly from the environment, BC suffers from performance limitations and poor generalization.

\textbf{Offline Reinforcement Learning.}  
In offline RL, regularization is commonly employed to mitigate distributional shift and stabilize policy learning from pre-collected datasets \cite{wu2019behavior, fujimoto2021minimalist, fujimoto2019off, kumar2019stabilizing, gao2025behavior}. These approaches typically incorporate a regularization term to constrain the policy updates and prevent overfitting to out-of-distribution actions. 

\textbf{Online Reinforcement Learning.}  
In online RL, regularization is mainly used to guide exploration and accelerate early learning \cite{pertsch2021accelerating, LiLWC0MC024Grasper, agarwal2022reincarnating}. Some studies \cite{pertsch2021accelerating} have explored adaptive adjustment of regularization strength, but most focus on single-agent settings, with limited research on multi-agent and graph-structured tasks.

\subsection{Graph Neural Networks in Multi-Agent Systems}
GNNs leverage message passing and aggregation on graph structures to effectively capture structural and global dependencies among nodes. Common architectures include GCN \cite{KipfW17GCN} and GAT \cite{velivckovic2017graph}. In multi-agent tasks, GNNs are widely used to model inter-agent interactions by representing the environment as a graph, where node features and adjacency relationships are jointly learned to improve policy generalization \cite{JiangDHL20DGN, mcclellan2024E2GN2, hu2023graph}.

Nevertheless, traditional GCN and GAT face challenges in highly dynamic, many-to-many interaction scenarios, where limited propagation range and inflexible information aggregation may hinder performance. Designing GNN architectures that adapt to dynamic interactions while integrating effectively with reinforcement learning remains a challenging and important research direction.

\section{Problem Formulation}

In realistic multi-agent scenarios, continuous spatial maps are often challenging to model directly, particularly in cooperative--competitive environments where agents interact under spatial and communication constraints. To address this, we discretize the continuous environment into an undirected graph structure, which reduces modeling complexity, enables efficient computation, and facilitates a more explicit representation of interaction patterns among agents. Let the environment be represented by $G = (V, E)$, where each node $v_i \in V$ denotes a feasible location that an agent can occupy, and each edge $e_{ij} \in E$ indicates that an agent can move between the two endpoint nodes. Obstacles in the original map restrict both movement and attack: if two agents are separated by obstacles, they cannot move toward or attack each other.

The studied task is modeled as a two-team zero-sum Markov game on a graph with $n$ nodes, where agents compete in either a pursuit or a confrontation scenario:
\begin{itemize}
    \item \textbf{Pursuit scenario:} The pursuer team consists of $m$ agents, while the evader team has a single agent. The game terminates when the shortest path distance between any pursuer and the evader is less than or equal to one.
    \item \textbf{Confrontation scenario:} Both teams contain $m$ homogeneous agents with identical initial HP. An agent is eliminated when its HP decreases to zero. The game ends when all agents on one team are eliminated.
\end{itemize}
The illustrations of both scenarios are shown in Figure~\ref{fig:scenario}.
\begin{figure}[h!]
  \centering
  \includegraphics[width=\linewidth]{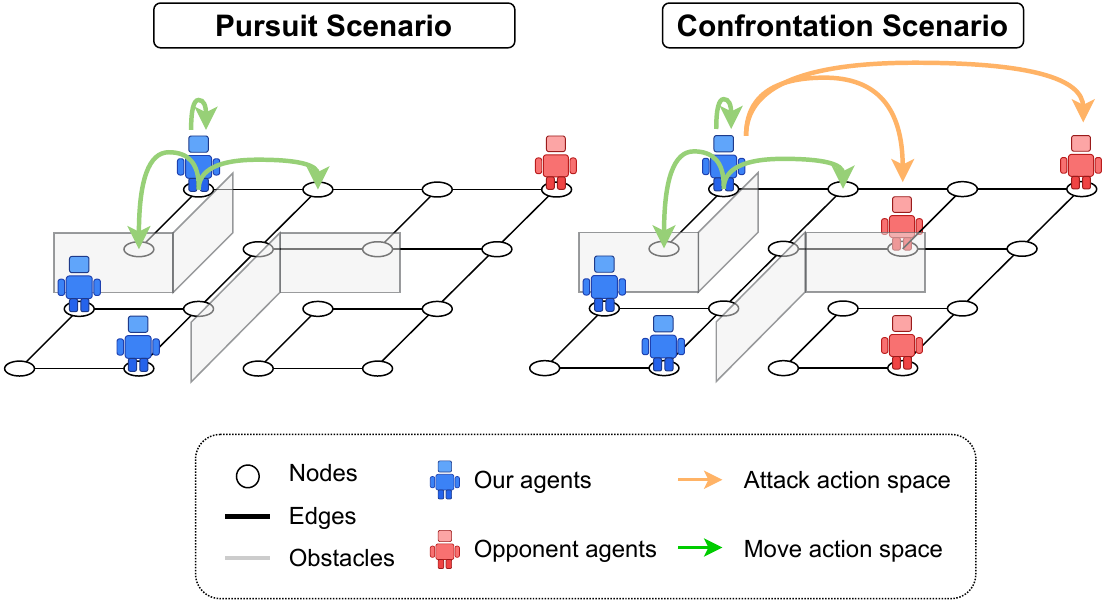}
  \caption{Illustration of pursuit and confrontation scenario.}
  \Description{Illustration of pursuit and confrontation scenario.}
  \label{fig:scenario}
\end{figure}

The Markov game is defined as $(\mathcal{S}, \mathcal{A}, \mathcal{B}, \mathcal{P}, \mathcal{R}, \gamma)$, where:
\begin{itemize}
    \item $\mathcal{S}$ is the global state space: in the pursuit scenario, it consists of all agents’ positions; in the confrontation scenario, it additionally includes each agent’s HP and survival status.  
    \item $\mathcal{A}$ is the action space of our agents, and $\mathcal{B}$ is that of the opponent. In the pursuit scenario, actions are limited to moving to a neighboring node; in the confrontation scenario, actions include moving to a neighboring node and attacking a surviving opponent agent.  
    \item $\mathcal{P}$ is the state transition probability describing the stochastic dynamics of the environment.  
    \item $\mathcal{R}$ is the reward function. In the pursuit scenario, a successful capture yields a reward of $r_\text{capture}$ to the pursuer team. In the confrontation scenario, eliminating an opponent agent yields a reward of $r_\text{kill}$, and eliminating all opponent agents yields an additional reward of $r_\text{all\_kill}$. Rewards for the opponent team are the negation of our team’s rewards due to the zero-sum nature of the game.
    \item $\gamma \in (0, 1]$ is the discount factor.
\end{itemize}

Let $\pi(\mathbf{a} \mid s)$ and $\nu(\mathbf{b} \mid s)$ denote the policies of our agents and the opponent agents, respectively. The objective of MARL is to find the optimal policy $\pi^*(\mathbf{a} \mid s)$ that maximizes the expected discounted cumulative reward over all joint actions of our agents, i.e., $\pi^* = \arg \max_{\pi} \mathbb{E}_{\pi} \left[ \sum_{t=0}^T \gamma^t r(s_t, \mathbf{a}_t) \right]$.

\section{Feature Representation and Attention-based Graph Encoding}
In multi-agent graph-structured tasks, dynamically evolving interactions greatly challenge feature extraction and policy optimization. Traditional GCN or GAT, though effective in static graphs, struggle in highly dynamic many-to-many scenarios, where restricted information propagation and rigid attention allocation lead to inadequate feature representation and suboptimal decision-making.

\begin{figure*}[h!]
  \centering
  \includegraphics[width=\linewidth]{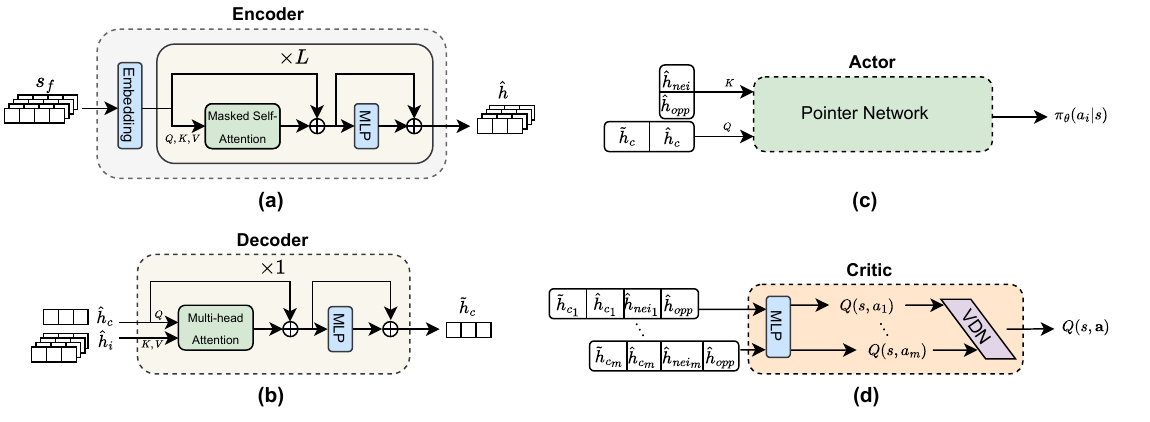}
  \caption{Structure of the feature representation graph encoding method.}
  \Description{Structure of the feature representation graph encoding method.}
  \label{fig:structure}
\end{figure*}

In the proposed MARL framework, both the actor and critic adopt an encoder--decoder architecture based on an attention mechanism for feature extraction. The attention mechanism enables each agent to dynamically adjust information weights from other nodes/agents according to the current environment state, thereby improving adaptability in highly dynamic multi-agent scenarios. The overall network structure is illustrated in Figure~\ref{fig:structure}. This section further details the construction of graph-level features and the attention-based encoder--decoder design.

\subsection{Graph-based State Representation}
Direct one-hot encoding of agent positions becomes inefficient on large graphs because it yields sparse features and does not reflect pairwise distances. We therefore represent an agent's position by the shortest-path distances from its current node to all nodes in the graph. The shortest-path distances are computed by the Floyd--Warshall algorithm. Note that these distances can be precomputed offline for static graphs; for dynamic graphs, they must be updated when the topology changes.

In the confrontation scenario, the global state further includes a damage-potential matrix that represents the potential attack damage between any pair of nodes, as well as each agent's HP and alive/dead flags. Before feeding into the network, we normalize features as follows: shortest distances are divided by the graph diameter (maximum finite distance), damage values are scaled by the maximum possible damage, and HP values are normalized by the initial HP. After normalization, these components are concatenated into a global feature representation and arranged into a feature matrix $s_f$, resulting in $s_f \in \mathbb{R}^{n \times f}$ where $n$ is the number of graph nodes and $f$ is the per-node feature dimensionality.

\subsection{Encoder}
The global feature matrix $s_f$ is projected by a linear layer to obtain node embeddings in $\mathbb{R}^{n\times d}$, which serve as the encoder input. The encoder contains $L$ masked self-attention layers (we used $L=6$ in our experiments). In each layer the query, key and value for node $i$ are computed as $q_i = W_Q h_i$, $k_i = W_K h_i$, $v_i = W_V h_i$, where $W_Q,W_K,W_V$ are trainable weight matrices and $h_i$ is the input embedding of node $i$. The raw attention score is $u_{ij} = q_i^\top k_j / \sqrt{d}$. To enforce graph locality we apply an adjacency mask in the softmax:
\( w_{ij} = \dfrac{A_{ij}\exp(u_{ij})}{\sum_{t=1}^{n} A_{it}\exp(u_{it})}, \)
where $A$ is the adjacency matrix (entries are 1 for reachable/neighbour nodes and 0 otherwise). The updated node embedding is then $h'_i = \sum_{j=1}^n w_{ij} v_j$. Each attention sub-layer is followed by a position-wise feed-forward network, residual connections and layer normalization to stabilize training. After stacking $L$ masked self-attention layers, we obtain the embedding representation $\hat{h}\in\mathbb{R}^{n\times d}$.

\subsection{Decoder}
The decoder produces an agent-centric contextual embedding for the current agent located at node $c$. Given the encoder output $\hat{h}$, we take the embedding of node $c$, denoted as $\hat{h}_c$, to generate the decoder query: $q = W_Q \hat{h}_c$. Meanwhile, for each agent node $i$, its embedding $\hat{h}_i$ is projected into the key and value representations: $k_i = W_K \hat{h}_i, v_i = W_V \hat{h}_i$. The decoder attention scores are $u_i = q^\top k_i / \sqrt{d}$ and the decoder attention weights are $w_i = \dfrac{\exp(u_i)}{\sum_{t=1}^{n}\exp(u_t)}$, producing the contextualized agent embedding $\tilde{h}_c = \sum_{i=1}^n w_i v_i$.

\subsection{Actor Design}
For generating action selection probabilities, we then concatenate $\tilde{h}_c$ with $\hat{h}_c$ to form the final query for the pointer module. The keys and values for the pointer are constructed from candidate targets: neighboring nodes $\hat{h}_{\mathrm{nei}}$ in the pursuit scenario, and both neighboring nodes $\hat{h}_{\mathrm{nei}}$ and opponents nodes $\hat{h}_{\mathrm{opp}}$ in the confrontation scenario. The pointer network \cite{VinyalsFJ15pointer} computes scores over candidate actions and produces a probability distribution (policy) over these actions; the agent's action is sampled (or chosen by argmax) from this distribution. This design makes the actor's decision explicitly agent-centric and allows focused aggregation of relevant global information.

\subsection{Critic Design}
For value estimation, the critic consumes the encoder and decoder outputs. We form per-agent local inputs by concatenating $\tilde{h}_c$, $\hat{h}_c$, $\hat{h}_{\mathrm{nei}}$ and  $\hat{h}_{\mathrm{opp}}$, pass them through per-agent critic Multilayer Perceptron (MLPs) to obtain individual value estimates $Q(s,a_i)$, and aggregate them via VDN \cite{sunehag2018value}: $Q(s,\mathbf{a}) = \sum_i Q(s,a_i)$. VDN preserves decentralizability while allowing centralized learning of cooperative strategies.

\section{Adaptive Regularized Multi-Agent Soft Actor-Critic (ARAC)}

\begin{figure*}[h!]
  \centering
  \includegraphics[width=0.68\linewidth]{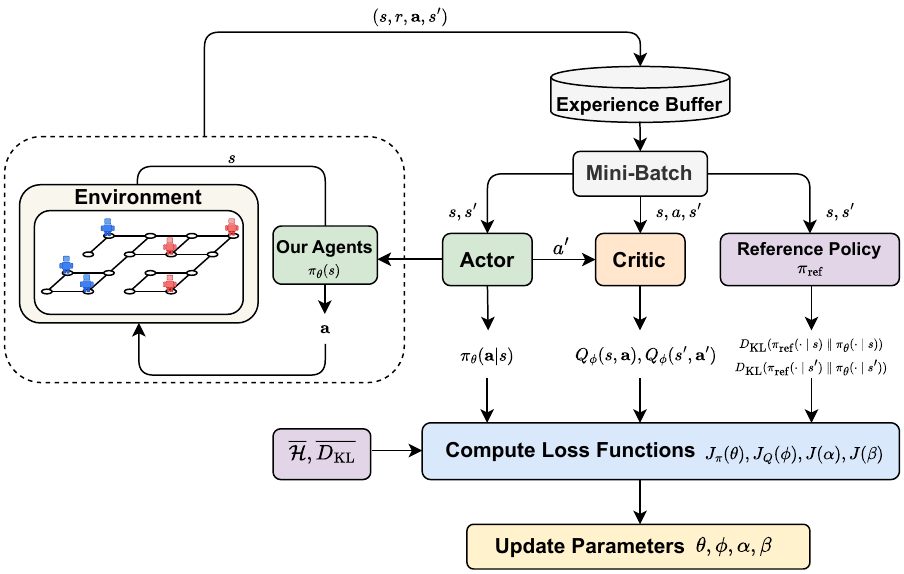}
  \caption{Overview of the ARAC training framework.}
  \label{fig:pipeline}
\end{figure*}

In MARL environments, sparse rewards pose a major challenge, as agents receive meaningful feedback only infrequently. This hinders accurate value estimation, slows convergence, increases the risk of suboptimal policies, and leads to inefficient exploration.  

A common solution is to introduce a fixed regularization weight $\beta$ that guides policy optimization towards a reference policy, thereby improving exploration. However, this approach suffers from two key limitations: (i) if the reference policy $\pi_{\mathrm{ref}}$ is suboptimal, a large $\beta$ causes over-reliance and restricts further improvement; (ii) the optimal $\beta$ is highly task-dependent, often requiring manual scheduling or decay, which is time-consuming and lacks generality. To address these issues, we propose the ARAC algorithm, whose overall pipeline is shown in Figure~\ref{fig:pipeline}.

\subsection{Modified Optimization Objective}
In traditional maximum entropy reinforcement learning, the policy is encouraged to maintain high entropy in order to balance exploration and exploitation. To address the sparse reward problem and further incorporate prior knowledge or expert demonstrations, we extend this formulation by introducing a reference policy $\pi_{\mathrm{ref}}$. This allows the learned policy $\pi$ to be regularized towards $\pi_{\mathrm{ref}}$ via a KL divergence term. The modified objective is:

\begin{equation}
\begin{aligned}
J(\pi) = \mathbb{E}_{\pi} \Big[ \sum_{t=0}^{\infty} r(s_t, \mathbf{a}_t) + \alpha \mathcal{H}(\pi(\cdot \mid s_t)) \\
 - \beta D_{\mathrm{KL}}\left( \pi_{\mathrm{ref}}(\cdot \mid s_t) \,\|\, \pi(\cdot \mid s_t) \right)\Big],
\end{aligned}
\end{equation}
where $\alpha$ controls the entropy weight, balancing exploration and exploitation, and $\beta$ controls the KL regularization strength, adjusting the influence of the reference policy. The KL term $D_{\mathrm{KL}}(\pi_{\mathrm{ref}}\|\pi)$ measures the deviation between the reference policy and the current policy. Unlike prior works (e.g., BRAC), we use $D_{\mathrm{KL}}(\pi_{\mathrm{ref}}(\cdot|s_t)\,\|\,\pi(\cdot|s_t))$ instead of $D_{\mathrm{KL}}(\pi(\cdot|s_t)\,\|\,\pi_{\mathrm{ref}}(\cdot|s_t))$, because when $\pi_{\mathrm{ref}}$ is deterministic, the latter KL term becomes undefined. 

Within the SAC framework, both the policy update and the value function update are modified to include the KL term.

For the policy update, the objective becomes:
\begin{equation}
\begin{aligned}
J_{\pi}(\theta) &= \mathbb{E}_{s_t \sim \mathcal{D}, \mathbf{a_t} \sim \pi_{\theta}} 
\left[ \alpha \log \pi_{\theta}(\mathbf{a}_t \mid s_t)
- Q_{\phi}(s_t, \mathbf{a}_t) \right. \\
&\quad \left. + \beta D_{\mathrm{KL}}\left( \pi_{\mathrm{ref}}(\cdot \mid s_t) 
\,\|\, \pi_{\theta}(\cdot \mid s_t) \right) \right],
\end{aligned}
\label{eq:policy_loss_full}
\end{equation}
where the first term encourages higher entropy, the second term drives the policy towards actions with higher estimated $Q$-values, and the third term enforces similarity to the reference policy.

The state value function is also adjusted accordingly. In standard SAC, the value function is defined as the expected $Q$-value under the policy minus the entropy term. Here, we additionally subtract the KL term:

\begin{equation}
\begin{aligned}
V(s_t)=&\mathbb{E}_{\mathbf{a}_t\sim\pi_\theta}\Big[Q_\phi(s_t,\mathbf{a}_t)-\alpha\log\pi_\theta(\mathbf{a}_t\mid s_t)\\
&-\beta D_{\mathrm{KL}}\left(\pi_{\mathrm{ref}}(\cdot\mid s_t)\parallel\pi_\theta(\cdot\mid s_t)\right)\Big].
\end{aligned}
\label{eq:value_update_full}
\end{equation}

This modification ensures that the critic also incorporates the influence of the reference policy, aligning the value estimates with the modified policy objective.

The critic follows SAC and maintains two Q-networks $Q_{\phi_1}, Q_{\phi_2}$ with corresponding target networks $Q_{\overline{\phi}_1}, Q_{\overline{\phi}_2}$. Using the minimum of the two target critics mitigates overestimation bias, while the soft updated target networks stabilize training:
\begin{equation}
\begin{aligned}
J_Q(\phi_i) = & 
\mathbb{E}_{(s_t, \mathbf{a}_t, r_t, s_{t+1}) \sim \mathcal{D}}
\Bigg[
\frac{1}{2} \Big( 
Q_{\phi_i}(s_t, \mathbf{a}_t) 
- \Big( r(s_t, \mathbf{a}_t) \\
&+ \gamma \big( \min_{j=1,2} Q_{\overline{\phi}_j}(s_{t+1}, \mathbf{a}_{t+1}) 
- 
\alpha \log \pi_{\theta}(\mathbf{a}_{t+1} \mid s_{t+1}) \\
&- \beta D_{\mathrm{KL}}(\pi_{\mathrm{ref}}(\cdot\mid s_{t+1}) \| \pi_{\theta}(\cdot\mid s_{t+1})) \big) 
\Big) \Big)^2
\Bigg].
\end{aligned}
\label{eq:critic_loss_full}
\end{equation}

\subsection{Theoretical Analysis}

To guarantee the soundness of the modified algorithm, we provide two key theoretical results: (1) the convergence of policy evaluation under the modified Bellman operator, and (2) the monotonic improvement property of the policy update.

First, we introduce a per-state regularizer (coming from entropy and the KL term) and denote it generically by: 
\[
\Omega_s(\pi)
:= \alpha \mathcal{H}\big(\pi(\cdot\mid s)\big)
- \beta\, D_{\mathrm{KL}}\big(\pi_{\mathrm{ref}}(\cdot\mid s)\,\|\,\pi(\cdot\mid s)\big),
\]
which depends only on the policy at state $s$ (and, importantly for the proofs below, does \emph{not} depend on the $Q$-function). We will assume that for any fixed policy $\pi$ the quantity $\Omega_s(\pi)$ is bounded uniformly in $s$; i.e. there exists a constant $C_{\Omega}$ such that: 
\[
\sup_{s\in\mathcal S} \big|\Omega_s(\pi)\big| \le C_{\Omega} < \infty.
\]

The regularized state-value and action-value functions for a fixed policy $\pi$ are defined as: 
\begin{align}
Q^{\pi}(s,\mathbf{a}) &= r(s,\mathbf{a}) + \gamma \mathbb{E}_{s' \sim p(\cdot\mid s,\mathbf{a})} \big[ V^{\pi}(s') \big], \label{eq:Qpi_def}\\
V^{\pi}(s) &= \mathbb{E}_{\mathbf{a}\sim\pi(\cdot\mid s)}\big[ Q^{\pi}(s,\mathbf{a}) \big] + \Omega_s(\pi). \label{eq:Vpi_def}
\end{align}

We define the policy-dependent Bellman operator $\mathcal{T}^{\pi}$ acting on any function $Q:\mathcal{S}\times\mathcal{A}\to\mathbb{R}$ by: 
\begin{equation}\label{eq:Tpi_def}
(\mathcal{T}^{\pi} Q)(s,\mathbf{a})
:= r(s,\mathbf{a}) + \gamma \mathbb{E}_{s' \sim p(\cdot\mid s,\mathbf{a})}
\Big[\,\mathbb{E}_{\mathbf{a}'\sim\pi(\cdot\mid s')}[Q(s',\mathbf{a}')] + \Omega_{s'}(\pi) \Big].
\end{equation}
Note that for the fixed policy $\pi$, the true $Q^\pi$ is the unique fixed point of $\mathcal{T}^\pi$.

\begin{theorem}[Policy Evaluation]\label{theorem:Evaluation}
Assume the reward $r(s,\mathbf{a})$ is bounded and the regularization term $\Omega_s(\pi)$ for a fixed policy $\pi$ is bounded. Then the operator $\mathcal{T}^\pi$ is a $\gamma$-contraction in the $\|\cdot\|_\infty$ norm. The iteration $Q \leftarrow \mathcal{T}^\pi Q$ converges uniformly to the unique fixed point $Q^\pi$ from any initial value.
\end{theorem}

\begin{proof}
See Appendix~\ref{proof:pe}
\end{proof}

\begin{theorem}[Policy Improvement]\label{theorem:pi}
Given any policy $\pi$, define:
\[
\pi^{\prime}(\cdot|s) = \arg\max_{\mu(\cdot|s)} \left\{ \mathbb{E}_{\mathbf{a}\sim\mu}[Q^\pi(s,\mathbf{a})] + \Omega_s(\mu) \right\}
\]
Then $V^{\pi^{\prime}}(s) \ge V^\pi(s)$ for all $s$, and $J(\pi^{\prime}) \ge J(\pi)$.
\end{theorem}

\begin{proof}
See Appendix~\ref{proof:pi}
\end{proof}

Based on the established results of policy evaluation and policy improvement, policy iteration is guaranteed to converge. Specifically, let the sequence of policies $\{\pi_k\}$ be generated by alternately performing policy evaluation to obtain $V^{\pi_k}$ and improving the policy greedily to obtain $\pi_{k+1}$. By the policy improvement theorem, for each state it holds that $V^{\pi_{k+1}} \ge V^{\pi_k}$, and the value function sequence is monotonically increasing and bounded, thus converging to a limit $\bar V$. In the case of a finite policy space, the sequence reaches a fixed policy $\pi^*$ in a finite number of steps, which is greedy with respect to $\bar V$ and satisfies the regularized Bellman optimality equation. Therefore, policy iteration converges and is valid.

\subsection{Adaptive Regularization Coefficients}
Following the spirit of the SAC, we formulate the reinforcement learning objective as a constrained optimization problem. For simplicity, we consider the undiscounted case with $\gamma=1$. Two types of constraints are imposed: (i) an entropy constraint that ensures sufficient exploration; and (ii) a KL-divergence constraint that prevents the learned policy from deviating excessively from a given reference policy. Formally, we write
\[\begin{aligned}
\max_{\pi_0,\ldots,\pi_T}\quad&\mathbb{E}\Bigg[\sum_{t=0}^T r(s_t,\mathbf{a}_t)\Bigg]\\
\text{s.t.} \quad \forall t,~~
&\mathbb{E}_{(s_t,\mathbf{a}_t)\sim\rho_\pi}\Big[-\log \pi_t(\mathbf{a}_t|s_t)\Big] \;\ge\; \overline{\mathcal{H}},\\
& D_{\mathrm{KL}}\big(\pi_{\mathrm{ref}}(\cdot|s_t)\,\|\,\pi_t(\cdot|s_t)\big) \;\le\; \overline{D_{\mathrm{KL}}}.
\end{aligned}\]

To solve this problem, we adopt a (approximate) dynamic programming approach and recursively optimize policies backward in time. The optimization objective can be expressed as
\[\max_{\pi_0}\Big(\mathbb{E}[r(s_0,\mathbf{a}_0)] 
+ \max_{\pi_1}\big(\mathbb{E}[\ldots]
+ \max_{\pi_T}\mathbb{E}[r(s_T,\mathbf{a}_T)]\big)\Big).\]

At the final step $T$, we employ the Lagrangian method with dual variables $\alpha_T \ge 0$ and $\beta_T \ge 0$, yielding
\[\begin{aligned}
\max_{\pi_T}\; \mathbb{E}[r(s_T,\mathbf{a}_T)]
&=\min_{\alpha_T\ge 0}\;\min_{\beta_T\ge 0}\;\max_{\pi_T} 
\Big\{ \mathbb{E}\big[ r(s_T,\mathbf{a}_T) \big]\\
&\quad + \alpha_T \big(-\log \pi_T(\mathbf{a}_T|s_T) - \overline{\mathcal{H}}\big) \\
&\quad + \beta_T \Big(\overline{D_{\mathrm{KL}}} - D_{\mathrm{KL}}\big[\pi_{\mathrm{ref}}(\cdot|s_T)\,\|\,\pi_T(\cdot|s_T)\big]\Big) \Big\}.
\end{aligned}\]
Since the objective is linear in $\pi_T$ and the constraint functions (entropy and KL divergence) are convex, strong duality holds. Once the optimal policy $\pi_T^*(\mathbf{a}_T|s_T)$ is obtained, the corresponding dual variables can be updated as
\[
\alpha_T^*=\arg\min_{\alpha_T\ge0}\;
\mathbb{E}_{(s_T,\mathbf{a}_T)\sim\rho_{\pi_T^*}}
\Big[-\alpha_T \log \pi_T^*(\mathbf{a}_T|s_T) - \alpha_T \overline{\mathcal{H}}\Big],
\]
\[
\beta_T^*=\arg\min_{\beta_T\ge0}\;
\mathbb{E}_{s_T\sim\rho_{\pi_T^*}}
\Big[\beta_T \overline{D_{\mathrm{KL}}} - \beta_T D_{\mathrm{KL}}\big(\pi_{\mathrm{ref}}(\cdot|s_T)\,\|\,\pi_T(\cdot|s_T)\big)\Big].
\]

We define the regularized soft $Q$-function as
\[
\begin{aligned}
Q_t^*(s_t,\mathbf{a}_t)
&= \mathbb{E}\big[r(s_t,\mathbf{a}_t)\big] + \mathbb{E}_{(s_{t+1},\mathbf{a}_{t+1})\sim\rho_{\pi^*}}\Big[
Q_{t+1}^*(s_{t+1},\mathbf{a}_{t+1}) \\
&- \alpha_{t+1}^* \log \pi_{t+1}^*(\mathbf{a}_{t+1}|s_{t+1}) \\
&- \beta_{t+1}^* D_{\mathrm{KL}}\big(\pi_{\mathrm{ref}}(\cdot|s_{t+1})\,\|\,\pi_{t+1}^*(\cdot|s_{t+1})\big)\Big],
\end{aligned}
\]
with $Q_T^*(s_T,\mathbf{a}_T)=\mathbb{E}\big[r(s_T,\mathbf{a}_T)\big].$

At time $T-1$, the optimization problem becomes
\[
\begin{aligned}
&\max_{\pi_{T-1}}\Big(
\mathbb{E}[r(s_{T-1},a_{T-1})] 
+ \max_{\pi_T}\mathbb{E}[r(s_T,a_T)]
\Big)\\
&= \max_{\pi_{T-1}}\Big(
Q_{T-1}^*(s_{T-1},a_{T-1})
- \alpha_T^*\mathcal{H}(\pi_T^*)
- \beta_T^* D_{\mathrm{KL}}(\pi_{\mathrm{ref}}\|\pi_T^*) \Big),
\end{aligned}
\]
subject to the same entropy and KL constraints. Applying the Lagrangian method again, we obtain
\[
\begin{aligned}
&\min_{\alpha_{T-1}\ge 0}\;\min_{\beta_{T-1}\ge 0}\;
\max_{\pi_{T-1}}\Big[
Q_{T-1}^*(s_{T-1},a_{T-1})
+ \alpha_{T-1}\big(\mathcal{H}(\pi_{T-1}) - \overline{\mathcal{H}}\big) \\
&\quad + \beta_{T-1}\big(\overline{D_{\mathrm{KL}}} - D_{\mathrm{KL}}(\pi_{\mathrm{ref}}\|\pi_{T-1})\big)\Big]
- \alpha_T^*\mathcal{H}(\pi_T^*)
- \beta_T^* D_{\mathrm{KL}}(\pi_{\mathrm{ref}}\|\pi_T^*).
\end{aligned}
\]
The dual variables $\alpha_{T-1}^*$ and $\beta_{T-1}^*$ are updated in the same way as at step $T$.

By recursion, for any step $t$ we obtain the optimal policy $\pi_t^*(\mathbf{a}_t|s_t)$ together with dual variable updates:
\[
\alpha_t^*=\arg\min_{\alpha_t\ge0}\;
\mathbb{E}_{(s_t,\mathbf{a}_t)\sim\rho_{\pi_t^*}}
\Big[-\alpha_t \log \pi_t^*(\mathbf{a}_t|s_t) - \alpha_t \overline{\mathcal{H}}\Big],
\]
\[
\beta_t^*=\arg\min_{\beta_t\ge0}\;
\mathbb{E}_{s_t\sim\rho_{\pi_t^*}}
\Big[\beta_t \overline{D_{\mathrm{KL}}} - \beta_t D_{\mathrm{KL}}\big(\pi_{\mathrm{ref}}(\cdot|s_t)\,\|\,\pi_t^*(\cdot|s_t)\big)\Big].
\]

According to dual gradient descent theory, the parameters $\alpha$ and $\beta$ should be updated after fully optimizing the primal variables. However, exact optimization is infeasible when these variables are represented by neural networks. Instead, we adopt a truncated approximation: after a few gradient updates on the policy, $\alpha$ and $\beta$ are directly updated via gradient descent on their dual objectives:

\begin{equation}\label{eq:alpha_update_full}
J(\alpha) = \mathbb{E}_{s_t \sim \mathcal{D},\;\mathbf{a}_t\sim\pi_\theta}
\Big[-\alpha \log \pi_\theta(\mathbf{a}_t|s_t) - \alpha \overline{\mathcal{H}}\Big],
\end{equation}
\begin{equation}\label{eq:beta_update_full}
J(\beta) = \mathbb{E}_{s_t \sim \mathcal{D}}
\Big[-\beta D_{\mathrm{KL}}(\pi_{\mathrm{ref}}(\cdot|s_t)\,\|\,\pi_\theta(\cdot|s_t)) + \beta \overline{D_{\mathrm{KL}}}\Big].
\end{equation}

Although this relaxation violates the convexity assumptions of classical duality theory, prior work \cite{haarnoja2018soft2} has shown that it remains stable and effective in practice, enabling $\alpha$ and $\beta$ to adaptively adjust entropy and regularization strengths during training.

\subsection{Algorithm Description}

The complete training procedure of ARAC is summarized in Algorithm~\ref{alg:arac}. The main differences from standard SAC are the modified value function update and the additional optimization of the regularization coefficient $\beta$.

\begin{algorithm}[h!]
\caption{Adaptive Regularized Multi-Agent Soft Actor-Critic (ARAC)}
\label{alg:arac}
\KwIn{Replay buffer $\mathcal{D}$, policy $\pi_{\theta}$, two critics $Q_{\phi_1}, Q_{\phi_2}$ with target networks $Q_{\bar{\phi}_1}, Q_{\bar{\phi}_2}$, reference policy $\pi_{\mathrm{ref}}$}
\KwOut{Optimized policy $\pi_{\theta}$}
Initialize learnable parameters $\alpha$ and $\beta$\;
\While{not converged}{
    Observe state $s$ and select joint action $\mathbf{a} \sim \pi_{\theta}(\cdot \mid s)$\;
    Execute $\mathbf{a}$, observe reward $r$ and next state $s'$\;
    Store $(s, \mathbf{a}, r, s')$ into $\mathcal{D}$\;
    Sample mini-batch from $\mathcal{D}$\;
    Update both critics $Q_{\phi_1}, Q_{\phi_2}$ using Eq.~\eqref{eq:critic_loss_full} with target networks $Q_{\bar{\phi}_1}, Q_{\bar{\phi}_2}$\;
    Update policy using Eq.~\eqref{eq:policy_loss_full}\;
    Adjust $\alpha$ using Eq.~\eqref{eq:alpha_update_full}\;
    Adjust $\beta$ using Eq.~\eqref{eq:beta_update_full}\;
    Soft update target networks $Q_{\bar{\phi}_1}, Q_{\bar{\phi}_2}$\;
}
\end{algorithm}

\section{Experiments}

\subsection{Experimental Settings}

\subsubsection{Scenarios and Map Configurations}
To evaluate the effectiveness of the proposed ARAC in graph-structured multi-agent tasks, we conduct experiments in two different scenarios: pursuit and confrontation.

\begin{itemize}
    \item \textbf{Pursuit Scenario}: The number of pursuers is set to \(m=2\). The experimental maps are selected from a set of 152 Dungeon maps, with the number of nodes ranging from 50 to 250\cite{lu2025equilibrium}. These maps have complex topologies with numerous loops and bottlenecks, which pose challenges for path planning and coordinated pursuit. The reward function is set to $r_\text{capture}=30$.
    \item \textbf{Confrontation Scenario}: The number of agents on each side is set to \(m=3\). A fixed map with 100 nodes is used, featuring multiple intersecting corridors and blocking regions, testing both offensive and defensive decision-making abilities. The reward function is set to $r_\text{kill}=3$, and $r_\text{all\_kill}=20$.
\end{itemize}

This design ensures that the pursuit scenario emphasizes dynamic target tracking and path planning, while the confrontation scenario focuses more on strategic game-playing and cooperative confrontation, providing a comprehensive evaluation of adaptability.

\subsubsection{Reference Policy Design}
Due to the sparse reward nature of the environments, we incorporate a reference policy as a reference signal for policy regularization. While not optimal, the reference policy significantly improves exploration efficiency in the early training stages.

\begin{itemize}
    \item \textbf{Pursuit Scenario}: The reference policy follows the shortest path to approach the current evader's position. Although computationally efficient, it may lead to endless chasing loops in maps containing cycles, resulting in failure.
    \item \textbf{Confrontation Scenario}: The reference policy is a rule-based attack-and-move strategy:  
    (1) If a living enemy agent is within a sensing range (set to 2 in our experiments), attack the closest living enemy;  
    (2) Otherwise, move along the shortest path toward the nearest living enemy.
\end{itemize}

While effective in small-scale engagements, these strategies lack the ability to optimize global tactics.

\subsection{Baseline Methods}
We compare ARAC against the following baselines:

\begin{enumerate}
    \item \textbf{BRAC \cite{wu2019behavior}}: An offline RL method that constrains the policy update with a divergence function, using a fixed coefficient \(\beta\). To match our setup, we adapt BRAC to an online version with KL divergence as the divergence function.
    \item \textbf{BC}: A pure supervised learning approach that fits the policy distribution to the reference policy without any RL loss.
    \item \textbf{Reference Policy}: The reference policy itself is used to directly interact with the environment, serving as a performance lower bound.
\end{enumerate}

This comparison design enables us to validate the effectiveness of adaptive \(\beta\) and distinguish its contribution from the reference policy alone.

\subsection{Algorithm Comparison }

\begin{figure}[h!]
	\centering    
	\subfigure[Pursuit scenario]{									\includegraphics[width=0.47\linewidth]{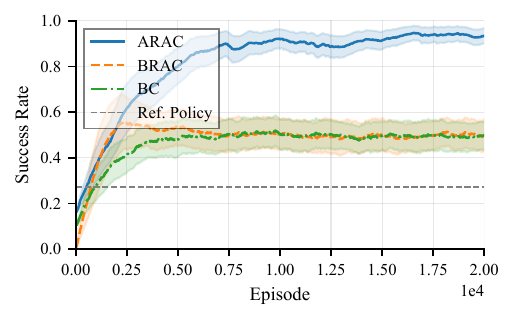}
    }
  \subfigure[Confrontation scenario]{									
		\includegraphics[width=0.47\linewidth]{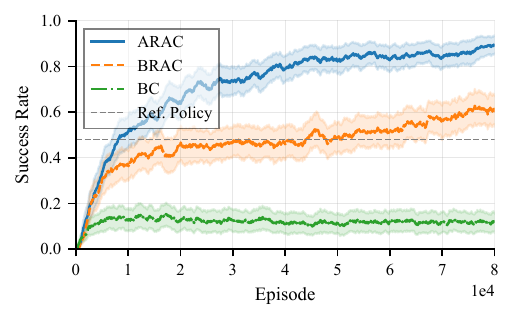}
        }
 
	\caption{Success rate curves of our method and baseline algorithms in pursuit and confrontation scenario. Shaded regions denote the standard deviation over 3 runs.}
 \label{fig:comparison}
\end{figure}

Figure~\ref{fig:comparison} present the success rate curves of different algorithms over training episodes in the pursuit and confrontation scenarios, respectively. In both tasks, the proposed ARAC method demonstrates clear advantages over all baselines.

In the \textbf{pursuit scenario}, the reference policy has a fixed success rate of $0.27$, indicating that it is suboptimal. ARAC achieves rapid performance improvement in the early stages of training, reaching a success rate close to $1.0$ within approximately $0.5\times 10^4$ to $0.75\times 10^4$ episodes, and maintaining stability thereafter. In comparison, BRAC gradually improves but converges more slowly and attains a slightly lower final success rate. BC, which relies solely on supervised learning to fit the reference policy, is limited by the quality of the reference policy itself, resulting in a relatively flat curve with limited improvement. 

In the \textbf{confrontation scenario}, the reference policy has a fixed success rate of $0.48$, also indicating that it is suboptimal. ARAC continues to significantly outperform other methods. It achieves a success rate close to $1.0$ within roughly $2\times 10^4$ episodes, with both faster convergence and higher stability compared to BRAC. BRAC’s convergence speed and final performance are again inferior to ARAC, while BC performs even more poorly in this scenario, showing minimal success rate improvement.

Overall, ARAC achieves both faster convergence and higher final success rates in both scenarios. This validates that the adaptive $\beta$ mechanism effectively exploits the reference policy for exploration guidance in early training, while progressively reducing reliance on it in later stages. Consequently, ARAC avoids the performance bottlenecks caused by suboptimal reference policies and facilitates the learning of superior strategies.

\subsection{Ablation Studies}
\subsubsection{Graph Neural Network Variants}
To analyze the impact of different graph representation methods on multi-agent cooperation and competition tasks, we compare against two widely used GNN architectures:

\begin{itemize}
    \item \textbf{GCN} \cite{KipfW17GCN}: Spectral convolution-based node aggregation with fixed adjacency matrices.
    \item \textbf{GAT} \cite{velivckovic2017graph}: Introduces attention weights for neighbors during message passing, enabling adaptive neighbor weighting.
\end{itemize}

\begin{figure}[h!]
	\centering    
	\subfigure[Pursuit scenario]{									\includegraphics[width=0.47\linewidth]{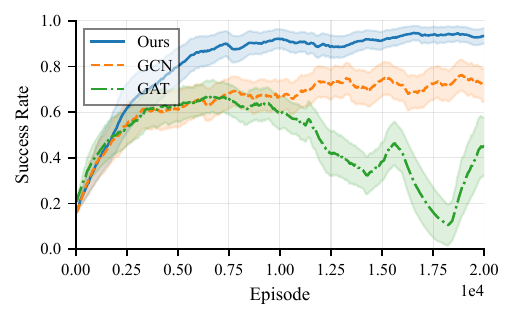}
    }
  \subfigure[Confrontation scenario]{									
		\includegraphics[width=0.47\linewidth]{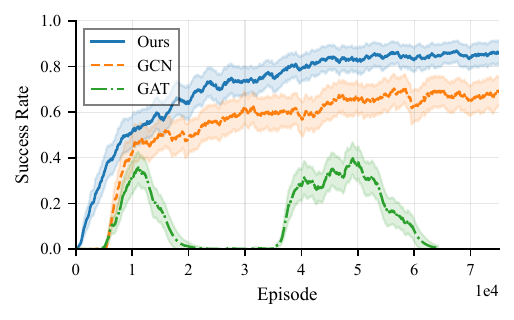}
        }
 
	\caption{Success rate curves of our method and other feature representation approaches in pursuit and confrontation scenario. Shaded regions denote the standard deviation over 3 runs.}
 \label{fig:gnn}
\end{figure}

Using ARAC as the training algorithm, we evaluate our method alongside GCN and GAT as graph-structured feature representation approaches. The training success rate curves for both the pursuit and confrontation scenarios are shown in Figure~\ref{fig:gnn}. The results clearly indicate that our method consistently outperforms the baselines in both tasks. As training progresses, our method’s success rate rapidly converges to a stable value close to $1.0$. In contrast, GCN plateaus during the mid-training phase, with its success rate hovering around $0.7$ in both scenarios, highlighting its limited learning capability. GAT performs the worst, exhibiting training collapse issues with severe curve oscillations throughout the training process, which reflects its instability.

\subsubsection{Scalability Experiments: Varying the Number of Agents}
A key challenge in MARL is the combinatorial explosion caused by increasing the number of agents, which leads to exponential growth in joint state and action spaces. To assess scalability, we extend the default \textit{3v3} confrontation scenario to \textit{4v4} and \textit{5v5} settings, evaluating changes in training stability and convergence speed.

\begin{figure}[h!]
  \centering
  \includegraphics[width=0.65\linewidth]{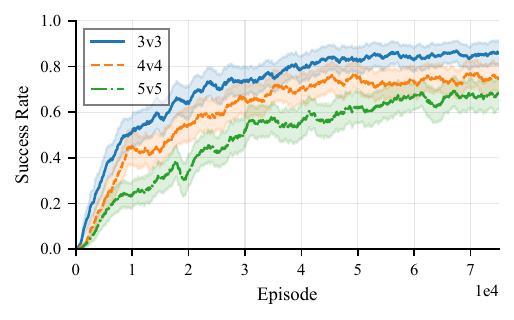}
  \caption{Success rate curves with varying numbers of agents in the confrontation scenario. Shaded regions denote the standard deviation over 3 runs.}
  \label{fig:agentnum}
\end{figure}

Using the ARAC algorithm, we trained agents in confrontation scenarios with varying numbers of agents, and the success rate curves are shown in Figure~\ref{fig:agentnum}. The experimental results demonstrate that ARAC can achieve stable training even with a larger number of agents, attaining success rates exceeding $0.6$.

\subsubsection{The Guidance of Reference Policy}

\begin{figure}[h!]
  \centering
  \includegraphics[width=0.65\linewidth]{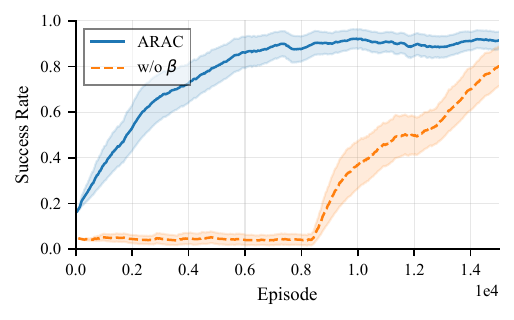}
  \caption{Success rate curves of ARAC and its variant without the reference policy term $\beta$ in the pursuit scenario. Shaded regions denote the standard deviation over 3 runs.}
  \label{fig:noBeta}
\end{figure}

To investigate the importance of the reference policy term $\beta$, we conduct an ablation study in the pursuit scenario, where the reward structure is more sparse compared to confrontation environments. We compare the proposed ARAC method with its variant that removes the reference policy regularization, denoted as \textit{w/o $\beta$}. The success rate curves are shown in Fig.~\ref{fig:noBeta}. 

It can be observed that \textit{w/o $\beta$} remains at a low success rate for the first $0.8\times10^4$ training steps before gradually increasing. In contrast, ARAC exhibits a rapid improvement from the beginning of training, reaching nearly perfect performance around $0.7\times10^4$ steps. This demonstrates that the reference policy term significantly accelerates learning in sparse-reward environments.

\section{Conclusion}
In this work, we introduced the ARAC framework for MARL on graph-structured pursuit and confrontation scenarios. ARAC integrates graph-based feature construction with an attention-driven encoder-decoder GNN to capture agents’ spatial relations and interaction dynamics. Furthermore, by adaptively adjusting the regularization weight in divergence-based policy optimization, ARAC is able to exploit reference policy guidance in the early stage of training while mitigating the risk of over-reliance in the long run. Extensive experiments on both pursuit and confrontation scenarios demonstrate that ARAC consistently outperforms strong baselines and maintains stable scalability with increasing agent numbers. In future work, we plan to extend ARAC to larger-scale multi-agent systems with dynamic graph topologies, and investigate its integration with multi-modal perception and decision-making to further improve adaptability and generalization.

\begin{acks}
This work was supported in part by the National Natural Science Foundation of China under Grants 62293541, 62136008, and 62206281, in part by Beijing Natural Science Foundation under Grant No 4232056, and in part by Beijing Nova Program under Grant 20240484514.
\end{acks}

\bibliographystyle{ACM-Reference-Format}
\bibliography{sample-base}

\appendix
\section{Proofs}
\subsection{Policy Evaluation}\label{proof:pe}
\begin{proof}
Fix arbitrary functions $Q_1,Q_2:\mathcal{S}\times\mathcal{A}\to\mathbb{R}$. For any state-action pair $(s,\mathbf{a})$ we have, by the definition \eqref{eq:Tpi_def},
\begin{align*}
&\big(\mathcal{T}^{\pi} Q_1\big)(s,\mathbf{a}) - \big(\mathcal{T}^{\pi} Q_2\big)(s,\mathbf{a}) \\
&\quad = \bigg( r(s,\mathbf{a}) + \gamma \mathbb{E}_{s'} \Big[ \mathbb{E}_{\mathbf{a}'\sim\pi} [Q_1(s',\mathbf{a}')] + \Omega_{s'}(\pi) \Big] \bigg) \\
&\qquad - \bigg( r(s,\mathbf{a}) + \gamma \mathbb{E}_{s'} \Big[ \mathbb{E}_{\mathbf{a}'\sim\pi} [Q_2(s',\mathbf{a}')] + \Omega_{s'}(\pi) \Big] \bigg) \\
&\quad = \gamma \, \mathbb{E}_{s' \sim p(\cdot\mid s,\mathbf{a})} \Big[ \mathbb{E}_{\mathbf{a}'\sim\pi} \big[ Q_1(s',\mathbf{a}') - Q_2(s',\mathbf{a}') \big] \Big].
\end{align*}

Taking absolute value and using the triangle inequality and linearity of expectations,
\begin{align*}
&\big| (\mathcal{T}^{\pi} Q_1)(s,\mathbf{a}) - (\mathcal{T}^{\pi} Q_2)(s,\mathbf{a}) \big| \\
&\quad = \gamma \left| \mathbb{E}_{s'}\Big[ \mathbb{E}_{\mathbf{a}'\sim\pi}\big[ Q_1(s',\mathbf{a}') - Q_2(s',\mathbf{a}') \big] \Big] \right| \\
&\quad \le \gamma \, \mathbb{E}_{s'}\Big[ \mathbb{E}_{\mathbf{a}'\sim\pi} \big| Q_1(s',\mathbf{a}') - Q_2(s',\mathbf{a}') \big| \Big].
\end{align*}

For any $(s',\mathbf{a}')$ we have: 
\[
\big| Q_1(s',\mathbf{a}') - Q_2(s',\mathbf{a}') \big| \le \|Q_1 - Q_2\|_\infty,
\]
therefore, 
\[
\mathbb{E}_{s'}\mathbb{E}_{\mathbf{a}'\sim\pi} \big| Q_1(s',\mathbf{a}') - Q_2(s',\mathbf{a}') \big|
\le \|Q_1 - Q_2\|_\infty.
\]

Combining the previous inequalities yields the pointwise bound: 
\[
\big| (\mathcal{T}^{\pi} Q_1)(s,\mathbf{a}) - (\mathcal{T}^{\pi} Q_2)(s,\mathbf{a}) \big|
\le \gamma \, \|Q_1 - Q_2\|_\infty,
\]
for every $(s,\mathbf{a})$. Taking supremum over $(s,\mathbf{a})$ gives: 
\[
\|\mathcal{T}^{\pi} Q_1 - \mathcal{T}^{\pi} Q_2\|_\infty \le \gamma \, \|Q_1 - Q_2\|_\infty.
\]

Because $0\le\gamma<1$, $\mathcal{T}^\pi$ is a contraction. The Banach fixed-point theorem implies that $\mathcal{T}^\pi$ has a unique fixed point $Q^\pi$ and iterating $Q_{k+1}=\mathcal{T}^\pi Q_k$ from any initial $Q_0$ converges exponentially (in the sup norm) to $Q^\pi$.

\end{proof}

\subsection{Policy Improvement}\label{proof:pi}
\begin{proof}

For any value function $V:\mathcal S\to\mathbb R$, define the operator corresponding to policy $\pi'$:
\[
\begin{aligned}
(\mathcal{T}^{\pi'} V)(s)
:= &\mathbb{E}_{\mathbf{a}\sim\pi'(\cdot\mid s)}\big[ r(s,\mathbf{a}) \big] + \Omega_s(\pi') \\
&+ \gamma \mathbb{E}_{\mathbf{a}\sim\pi'(\cdot\mid s)}\big[ \mathbb{E}_{s'\sim p(\cdot\mid s,\mathbf a)}[ V(s') ] \big].
\end{aligned}
\]
Note that the true regularized value $V^{\pi'}$ is the unique fixed point of $\mathcal{T}^{\pi'}$ (by the same contraction argument as in Theorem~\ref{theorem:Evaluation} applied to this operator).

By the definition of $\pi'$ we have, for every state $s$,
\[
\mathbb{E}_{\mathbf{a}\sim\pi'}\big[ Q^\pi(s,\mathbf{a}) \big] + \Omega_s(\pi')
\;\ge\;
\mathbb{E}_{\mathbf{a}\sim\pi}\big[ Q^\pi(s,\mathbf{a}) \big] + \Omega_s(\pi)
= V^\pi(s),
\]
where the last equality is the definition of $V^\pi(s)$ (see \eqref{eq:Vpi_def}). Rearranging, this gives:
\[
\mathbb{E}_{\mathbf{a}\sim\pi'}\big[ Q^\pi(s,\mathbf{a}) \big] + \Omega_s(\pi') - V^\pi(s) \ge 0.
\]
Since $Q^\pi(s,\mathbf{a}) = r(s,\mathbf{a}) + \gamma \mathbb{E}_{s'}[V^\pi(s')]$ (by \eqref{eq:Qpi_def}), for every state $s$,
\[
\begin{aligned}
(\mathcal{T}^{\pi'} V)(s)
:= &\mathbb{E}_{\mathbf{a}\sim\pi'(\cdot\mid s)}\big[ r(s,\mathbf{a}) \big] + \Omega_s(\pi') \\
&+ \gamma \mathbb{E}_{\mathbf{a}\sim\pi'(\cdot\mid s)}\big[ \mathbb{E}_{s'\sim p(\cdot\mid s,\mathbf a)}[ V(s') ] \big] \\
= &\mathbb{E}_{\mathbf{a}\sim\pi'}\big[ Q^\pi(s,\mathbf{a}) \big] + \Omega_s(\pi')\\
\ge& V^\pi(s).
\end{aligned}
\]
This implies that applying the operator $\mathcal{T}^{\pi'}$ once to the old value function $V^\pi$ yields a function that is pointwise no smaller than $V^\pi$. Furthermore, the operator $\mathcal{T}^{\pi'}$ is monotone: if $V_1\ge V_2$ pointwise then $\mathcal{T}^{\pi'} V_1 \ge \mathcal{T}^{\pi'} V_2$. This follows because the operator is an expectation of $V$ multiplied by nonnegative weights and plus statewise terms independent of $V$.

Combining monotonicity, we obtain:
\[
V^\pi \le \mathcal{T}^{\pi'} V^\pi
\le \big(\mathcal{T}^{\pi'}\big)^2 V^\pi
\le \cdots
\]
That is, the sequence $\{ (\mathcal{T}^{\pi'})^k V^\pi \}_{k\ge 0}$ is pointwise nondecreasing.

Since $\mathcal{T}^{\pi'}$ is a $\gamma$-contraction (by the same argument as in Theorem~\ref{theorem:Evaluation} with fixed policy $\pi'$), the sequence $(\mathcal{T}^{\pi'})^k V^\pi$ converges pointwise (and uniformly) to the unique fixed point $V^{\pi'}$. Therefore, taking the limit in the monotone chain above yields:
\[
V^\pi \le V^{\pi'} \quad\text{pointwise on } \mathcal{S}.
\]
This establishes $V^{\pi'}(s) \ge V^\pi(s)$ for every state $s$.

Finally, the global objective $J(\pi)$ equals the expectation of $V^\pi$ under the initial state distribution $\rho_0$, i.e. $J(\pi) = \mathbb{E}_{s_0\sim\rho_0}[V^\pi(s_0)]$. Hence $V^{\pi'}\ge V^\pi$ implies $J(\pi')\ge J(\pi)$.

\end{proof}

\section{Implementation Details}
The experiments were conducted on a workstation equipped with four NVIDIA RTX 6000 Ada Generation GPUs (48 GB each) and an Intel Xeon Platinum 8358 64-core Processor CPU with 1 TB RAM. The software environment consisted of Ubuntu 20.04, Python 3.10 and CUDA 12.2. Table~\ref{tab:hyperparameters} summarizes the hyperparameter settings used in all experiments.

\begin{table}[htbp]
\centering
\caption{Hyperparameter Settings in Experiments}
\label{tab:hyperparameters}
\begin{tabular}{l c}
\hline
\textbf{Hyperparameter} & \textbf{Value} \\
\hline
Optimizer & Adam \\
Batch Size & 128 \\
Replay Buffer Size & 2000 \\
Learning Rate & $10^{-5}$ \\
Max Time Steps & 128 \\
Encoder Layers & 6 \\
Decoder Layers & 1 \\
Number of Attention Heads & 8 \\
Target Entropy $\overline{\mathcal{H}}$ & $0.05 \times \log(\mathrm{dim}(\mathcal{A}))$ \\
Target KL Divergence $\overline{D_{\mathrm{KL}}}$ & 1 \\
\hline
\end{tabular}
\end{table}

\section{The Cross-Graph Generalizability}

To evaluate the generalization capability of the proposed method across different graphs, we conducted a cross-map generalization experiment. In this experiment, we selected three distinct confrontation maps (as shown in Fig.~\ref{fig:maps}), which differ significantly in terms of terrain layout, obstacle distribution, and initial position settings. These variations allow us to effectively assess the adaptability of agents to diverse spatial configurations.

\begin{figure}[h!]
\centering    
\subfigure[Map 1]{								\includegraphics[width=0.5\linewidth]{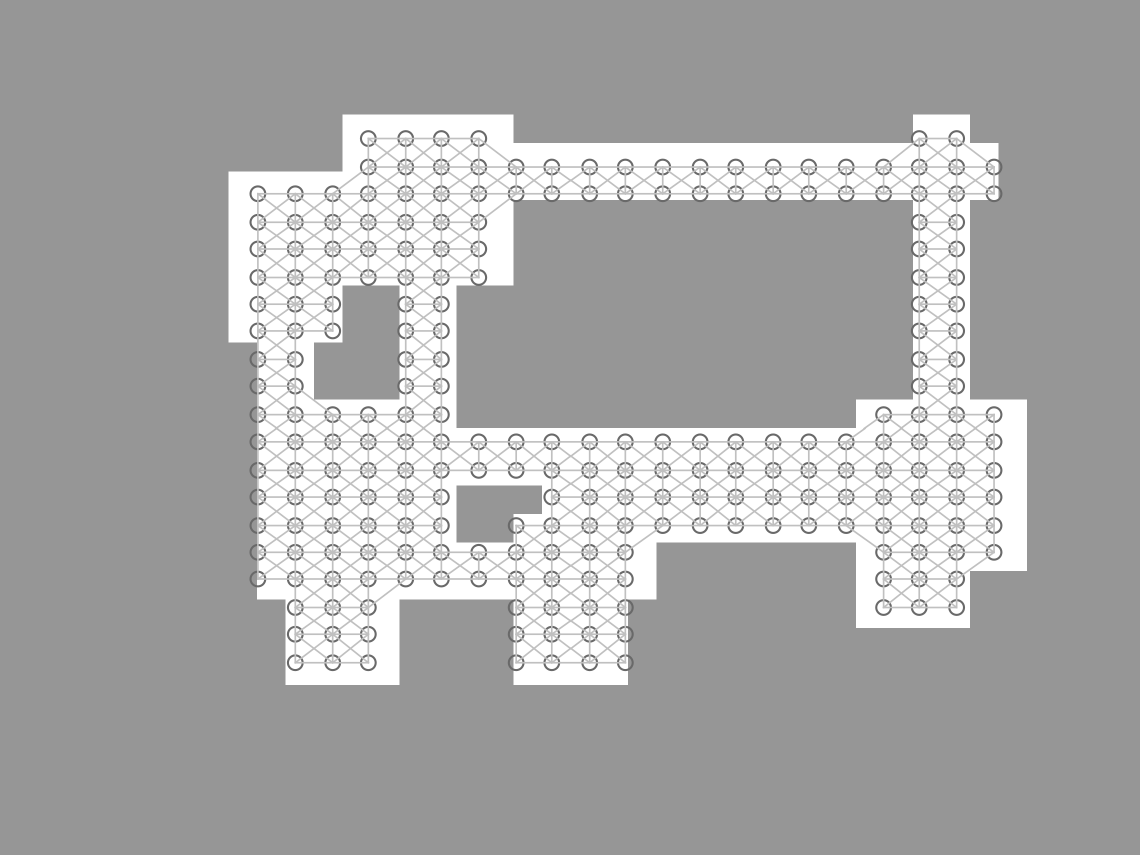}
}
\subfigure[Map 2]{\includegraphics[width=0.5\linewidth]{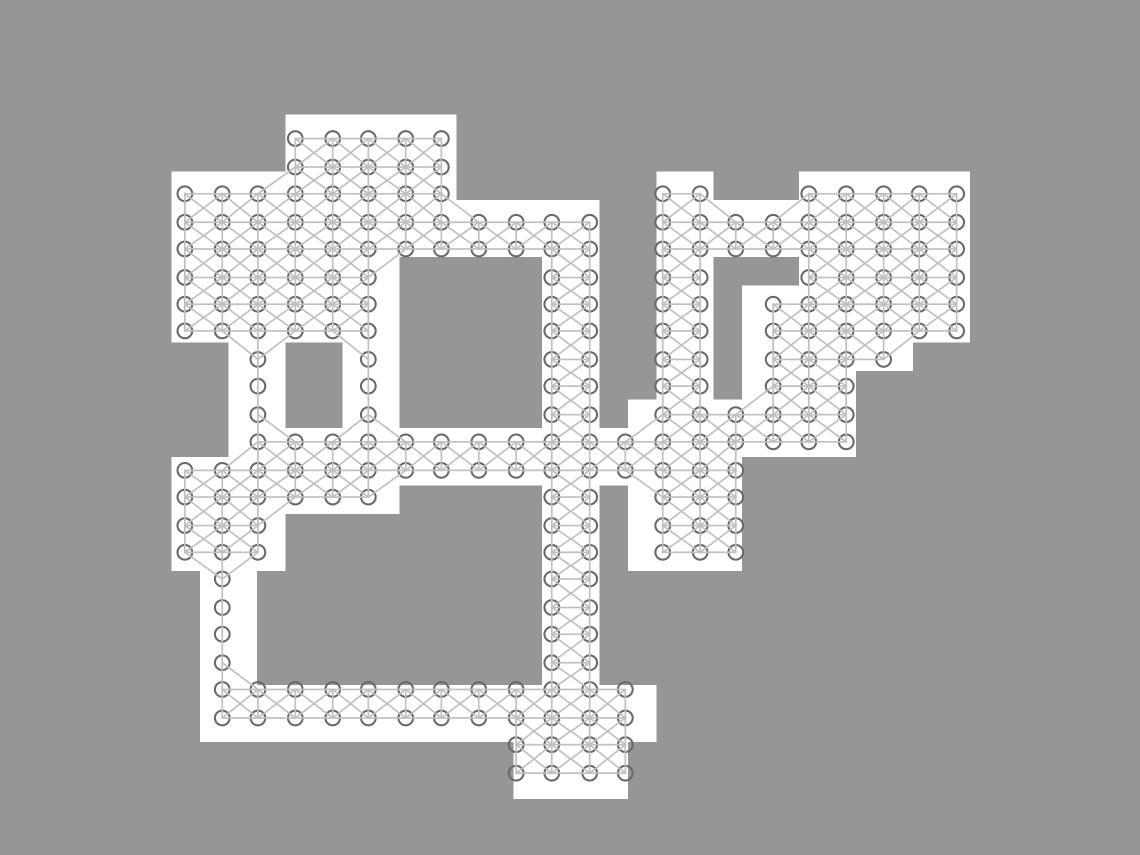}
}
\subfigure[Map 3]{
\includegraphics[width=0.5\linewidth]{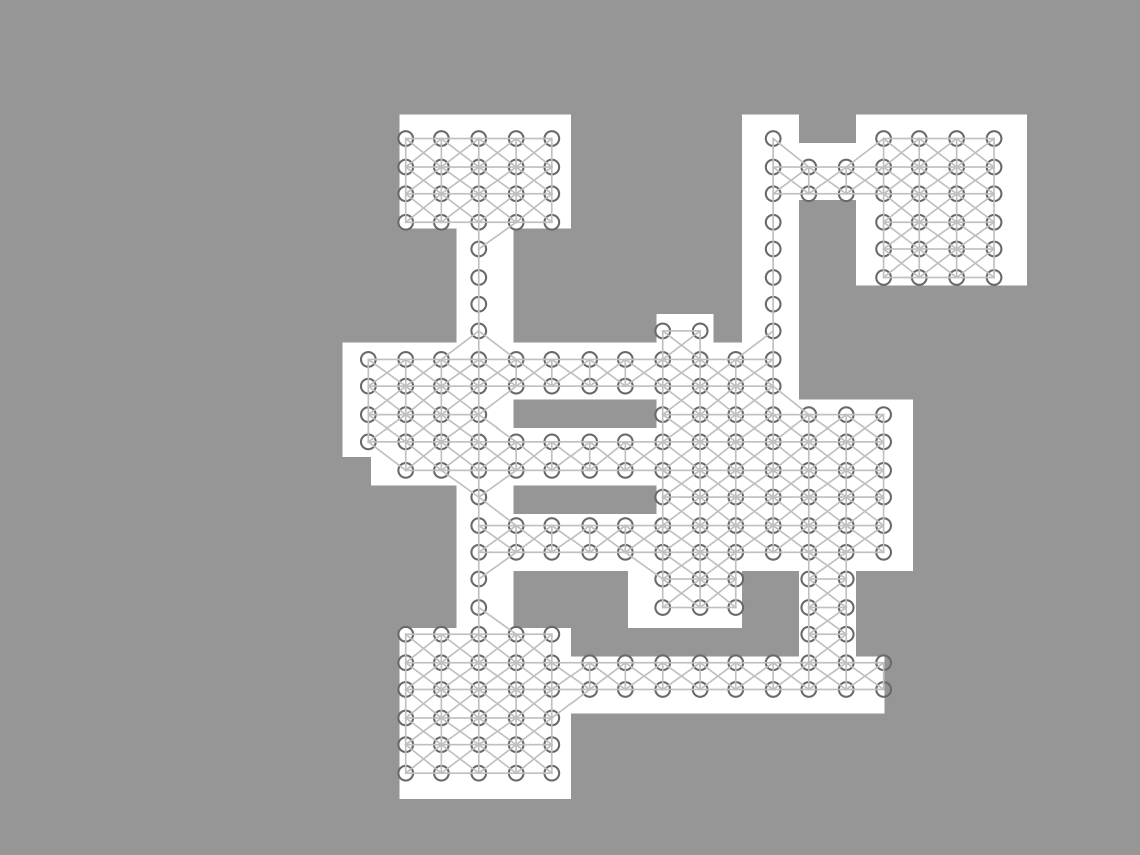}
}

\caption{Three maps used to verify cross-graph generalization.}
\label{fig:maps}
\end{figure}

For each map, we trained the agents independently and, after training, evaluated the resulting models on all three maps. This yielded a training–testing success rate matrix, as presented in Table~\ref{tab:cross-map}.

\begin{table}[h!]
\centering
\caption{Success rates of the cross-map generalization experiment.}
\renewcommand{\arraystretch}{1.2} 
\setlength{\tabcolsep}{8pt} 
\begin{tabular}{lccc}
\toprule
\multirow{2}{*}{\textbf{Test Map}} & \multicolumn{3}{c}{\textbf{Train Map}} \\
\cmidrule(lr){2-4}
 & \textbf{Map 1} & \textbf{Map 2} & \textbf{Map 3} \\ 
\midrule
\textbf{Map 1} & \textbf{0.912} & 0.910 & 0.890 \\
\textbf{Map 2} & 0.834 & \textbf{0.852} & 0.850 \\
\textbf{Map 3} & 0.766 & 0.750 & \textbf{0.778} \\ 
\bottomrule
\end{tabular}
\label{tab:cross-map}
\end{table}

The results show that the models maintain a relatively high success rate even on \textbf{unseen maps}. For example, the model trained on Map 1 achieved success rates of $0.910$ and $0.890$ on Map 2 and Map 3, respectively, only slightly lower than its performance on the training map. This demonstrates that the proposed method not only learns high-quality policies in the training environment but also transfers effectively to unseen maps, indicating strong cross-graph generalization.

\section{Self-Play for Further Performance Improvement in Confrontation Scenario}

Self-play is a widely used technique in reinforcement learning for training agents in competitive or adversarial environments. The key idea is that instead of relying solely on external opponents, the agent learns by continuously playing against versions of itself. This approach provides an ever-evolving training signal: as the agent improves, its opponent also becomes stronger, thereby creating an implicit curriculum without the need for manually designed adversaries. Self-play has been successfully applied in several landmark achievements, demonstrating its ability to produce highly robust and adaptive policies.

One of the main advantages of self-play is that it prevents overfitting to fixed opponents or scripted behaviors. Since the opponent is dynamically updated with the agent’s own policy, the learning process naturally fosters adaptability to diverse strategies and encourages the emergence of increasingly sophisticated tactics. Furthermore, self-play enables policies to generalize better to previously unseen opponents, making it especially suitable for confrontation scenarios where agents must deal with unpredictable adversarial behaviors.

In confrontation scenarios, self-play can be employed to further enhance policy performance beyond training solely against scripted or fixed opponents. Starting from the baseline policy $\pi_0$ trained with ARAC, we adopt a self-play procedure in which the agent continually plays against itself to progressively improve. This iterative adaptation ensures that the agent is not restricted to defeating a fixed set of adversaries but instead learns strategies that remain effective against increasingly stronger opponents.

To stabilize training and avoid oscillations in learning dynamics, we update the opponent policy with the latest agent policy every 100 episodes rather than at every step. This mechanism provides a moving target for the learning agent, encouraging consistent progress while preventing catastrophic forgetting of previously successful strategies. Figure~\ref{fig:sp-curve} illustrates the success rate of the latest policy when competing against the initial policy $\pi_0$ throughout the course of self-play.

\begin{figure}[h!]
  \centering
  \includegraphics[width=0.65\linewidth]{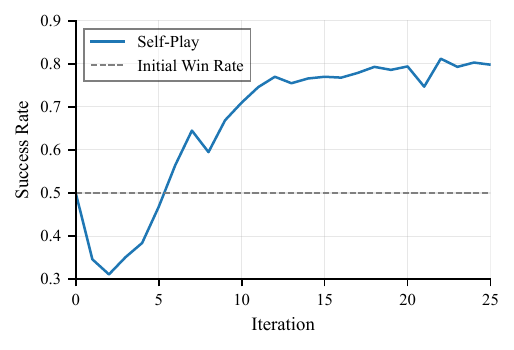}
  \caption{Success rate curves of the latest self-play policy against the initial policy $\pi_0$ in the confrontation scenario.}
  \label{fig:sp-curve}
\end{figure}

At the beginning of self-play, the win rate is approximately $0.5$, since both the agent and its opponent share the same policy. After around 25 iterations of policy updates, the latest policy achieves a win rate of nearly $0.8$ against $\pi_0$. This result demonstrates that self-play effectively improves policy performance, enabling the learned policy to adapt beyond defeating only scripted opponents and increasing its robustness to diverse adversarial behaviors.

To further validate the performance improvement during self-play, we conduct a pairwise evaluation among policies obtained from different iterations of the self-play process. Specifically, models from earlier and later iterations are matched against one another, and their win rates are recorded. The aggregated results are visualized as a heatmap in Figure~\ref{fig:sp-heatmap}. The clear upward trend in win rate along the iteration axis confirms that later policies consistently outperform earlier ones, verifying the effectiveness of the self-play training paradigm.

\begin{figure}[h!]
  \centering
  \includegraphics[width=0.6\linewidth]{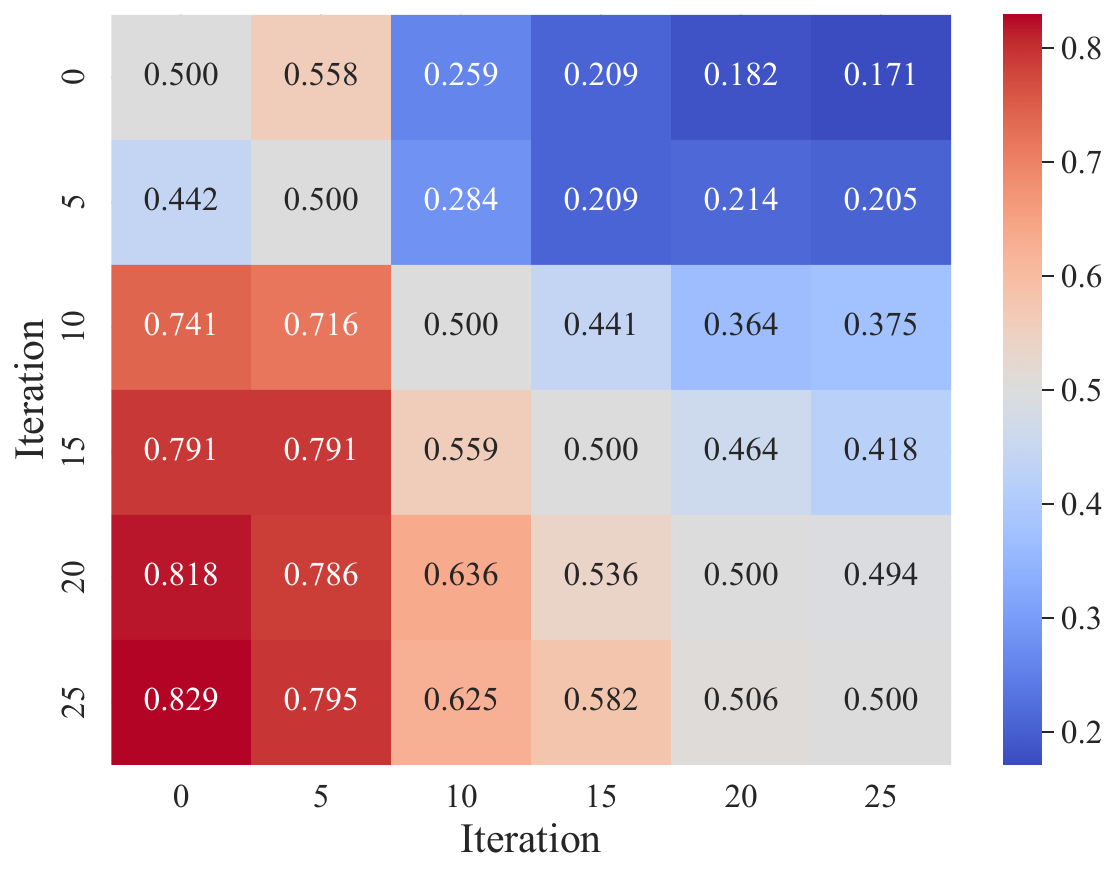}
  \caption{Pairwise win rate heatmap between policies from different self-play iterations.}
  \label{fig:sp-heatmap}
\end{figure}

Beyond quantitative improvements, self-play also provides several qualitative benefits. First, it diversifies the set of strategies that the agent encounters during training, mitigating overfitting to specific scripted opponents. Second, it encourages the emergence of novel behaviors through an implicit curriculum: as the opponent becomes stronger, the agent must discover more advanced tactics to maintain or improve its performance. Finally, by repeatedly adapting against its own evolving policy, the agent achieves a higher degree of robustness and generalization, which is crucial in dynamic multi-agent confrontation environments.

\end{document}